\documentclass[conference]{IEEEtran}
\usepackage{cite}
\usepackage{amsmath,amssymb,amsfonts, amsthm, dsfont}
\usepackage[ruled,vlined]{algorithm2e}
\usepackage{graphicx}
\usepackage{textcomp}
\usepackage{hyperref}
\usepackage{float}
\usepackage{xcolor}
\def\BibTeX{{\rm B\kern-.05em{\sc i\kern-.025em b}\kern-.08em
    T\kern-.1667em\lower.7ex\hbox{E}\kern-.125emX}}
\usepackage[caption=false]{subfig}
\newcommand{\bd}[1]{\boldsymbol{#1}}
\IEEEoverridecommandlockouts

\newtheorem{theorem}{Theorem}

\begin{document}

\title{Dirichlet process mixture models for non-stationary data streams\\
\thanks{This work has been partially supported by the Basque Government through the BERC 2022-2025 program and the Basque Modelling Task Force project; and by the Ministry of Science, Innovation and Universities: BCAM Severo Ochoa accreditation  SEV-2017-0718.
We are grateful to Andrés Masegosa for his helpful commentaries and suggestions.
}
}

\author{\IEEEauthorblockN{1\textsuperscript{st} Ioar Casado}
\IEEEauthorblockA{\textit{Basque Center for Applied Mathematics} \\
Bilbao, Spain \\
icasado@bcamath.org}
\and
\IEEEauthorblockN{2\textsuperscript{nd} Aritz Pérez}
\IEEEauthorblockA{\textit{Basque Center for Applied Mathematics} \\
Bilbao, Spain \\
aperez@bcamath.org}}

\maketitle

\begin{abstract}
In recent years we have seen a handful of work on inference algorithms over non-stationary data streams. Given their flexibility, Bayesian non-parametric models are a good candidate for these scenarios. However,  reliable streaming inference under the concept drift phenomenon is still an open problem for these models. In this work, we propose a variational inference algorithm for Dirichlet process mixture models. Our proposal deals with the concept drift by including an exponential forgetting over the prior global parameters. Our algorithm allows to adapt the learned model to the concept drifts automatically. We perform experiments in both synthetic and real data, showing that the proposed model is competitive with the state-of-the-art algorithms in the density estimation problem, and it outperforms them in the clustering problem. 
\end{abstract}


\section{Introduction}
Bayesian non-parametric (BNP) models have become a successful approach for dealing with increasingly complex data, and when it comes to density estimation and clustering, Dirichlet process mixture (DPM) models are the best known BNP models \cite{ferguson1973bayesian}, \cite{teh2010dirichlet}. In contrast with finite mixture models or standard clustering methods, in DPMs the number of mixture components (or clusters) adjusts to the complexity of available data. Apart from avoiding model selection problems, this property makes them specially suited for working with data streams, where data batches arrive sequentially and models need to adapt to the characteristics of the new data. 

Given the ubiquity of non-stationary phenomena in real life data streams, concept drift adaptation has seen great progress in the last decade \cite{lu2018learning}. However, advances in that area have been rarely combined with BNP models, hindering their real life applications \cite{reviewdrift_2020}. Even if there are effective streaming inference algorithms for DPMs, the majority of them implicitly assumes that the data stream is stationary. In order to fill this gap, we propose a new streaming variational inference (VI) algorithm for DPMs that can deal with concept drift.

\subsection*{Contributions}
In this work, we propose a streaming VI algorithm for DPM models, which, for the first time, extends Bayesian parametric forgetting methods \cite{masegosa2017bayesian} to the non-parametric case. This approach combines flexible adaptation to drifts of different magnitudes with the data-driven model complexity of BNPs. 

We perform experiments on both synthetic and real data streams. The experiments evaluate the learned Dirichlet process Gaussian mixtures from the density estimation and clustering points of view. We also analyze our model's ability to track the underlying parameters. The experimental results show that our model outperforms the state-of-the-art algorithms, especially in non-stationary environments.

\section{Related work}\label{rel_work}
The main challenge when working with DPMs and BNP models, in general, is to find efficient learning methods. Markov Chain Monte Carlo (MCMC) methods have been the basic approach for inference in DPMs \cite{neal2000markov}, but they have scalability problems for big datasets \cite{li2021review}, \cite{blei2017variational}. In \cite{attias1999variational} and \cite{jordan1999introduction} the authors introduced VI algorithms, which conceived posterior inference as an optimization problem \cite{bishop2006pattern}, providing faster approximate inference. This framework was adapted for DPMs by \cite{blei2006variational}. Since then, streaming versions of VI have been widely studied. 

Two main paradigms exist to tackle the problem of streaming VI: Streaming variational inference (SVB) \cite{broderick2013streaming} and stochastic variational inference (SVI) \cite{hoffman2013stochastic}. SVB updates the priors for batch $t$ with the posterior obtained from batch $t-1$. By initializing priors with the previous variational distribution, SVB implicitly assumes data interchangeability and is not adequate for non-stationary streams. The same limitation hinders the performance of more recent streaming VI methods such as \cite{NIPS2015_38af8613}.
SVI, on the other hand, extends gradient based optimization to VI. It is not exactly a streaming algorithm, but assumes we can access a fixed data set in an online fashion using minibatches. This requires to know the size of the dataset, $N$, beforehand, which is not feasible in a streaming scenario. However, this problem can be partially circumvented by manually selecting a value for $N$. More recently, sampling-based inference methods have been proposed for DPMs, which can deal with non-stationary data streams \cite{dinari2022sampling}, but there is still room for improvement among VI methods.

In order to obtain an effective VI algorithm for non-stationary DPMs, we propose a version of SVB with a forgetting mechanism. The global prior distribution for batch $t$ is a combination of i) the initial uniformative prior and ii) an exponential forgetting \cite{kulhavy1996duality} of the global variational distribution obtained after observing batch $t-1$. In the proposed procedure, the forgetting parameter controls how much we forget or retain from previous batches. This forgetting parameter can be learned automatically in a Bayesian manner with the hierarchical power priors method \cite{masegosa2017bayesian}. We propose a flexible learning approach with a hierarchical power for each parameter of the mixture model. By doing so, the components of the mixture have its own unique dynamic.

\section{Preliminaries}\label{preliminaries}
\subsection{Dirichlet process mixtures}

The Dirichlet process (DP) is a distribution over probability measures, hence draws from a DP are random distributions. It is often used as a non-parametric prior for infinite mixture models \cite{antoniak1974mixtures, escobar1995bayesian}.
Let $G_0$ be a distribution over the sample space $\Theta$ and let $\alpha$ be a positive scalar. A random distribution $G$ with the same support as $G_0$ is distributed according to a DP with the \textit{concentration parameter} $\alpha$ and the \textit{base distribution} $G_0$, i.e., $G\sim \text{DP}(\alpha, G_0)$, if for any finite measurable partition $\{B_1,\ldots,B_k\}$ of $\Theta$,
\begin{equation}\label{Def.DP}
(G(B_1),\ldots, G(B_k))\sim \text{Dir}(\alpha G_0(B_1),\ldots, \alpha G_0(B_k)).
\end{equation}

DPs were first introduced by Ferguson in \cite{ferguson1973bayesian}, where he proved their existence and basic consistency properties. $G_0$ is known as base distribution because, for any measurable $B\subseteq \Theta$ and any $G\sim\text{DP}(\alpha, G_0)$, we have $\mathbb{E}_{\text{DP}}[G(B)] = G_0(B)$. The concentration parameter $\alpha$ controls the probability mass around the mean of the DP, as $G\rightarrow G_0$ pointwise when $\alpha\rightarrow \infty$.

Samples from $\text{DP}(\alpha, G_0)$ are discrete (i.e. infinite sums of point masses) with probability one. Furthermore, sampling from any $G\sim\text{DP}(\alpha, G_0)$ exhibits a clustering property: let $\theta_1,\ldots,\theta_n$ be samples from $G$. The predictive distribution for $\theta_{n+1}$ marginalizing $G$ is 
\begin{equation}\label{predictiveDP}
        p(\theta_{n+1}|\theta_1,\ldots,\theta_n) = \frac{1}{\alpha+n}\Big(\alpha G_0 + \sum_{i=1}^n \delta_{\theta_i}\Big),
\end{equation}
\noindent where $\delta_{\theta_i}$ is the point mass located at $\theta_i$. Thus, successive draws from $G$ will take repeated values with positive probability in a rich-get-richer fashion. We refer to \cite{teh2010dirichlet} for a complete survey of the basic properties of DPs.

The discreteness and clustering properties of any $G\sim \text{DP}(\alpha, G_0)$ uphold the Dirichlet process as a non-parametric prior for the global parameters of infinite mixture models. The stick-breaking construction of DPs given by  \cite{sethuraman1994constructive} takes this intuition further. For $k\geq 1$, we define

\begin{align}\label{stickb}
    &\beta_k \sim \mathcal{B}(\cdot|1,\alpha),
    &\theta_k \sim G_0\nonumber,\\
    &\pi_k = \beta_k\prod_{t=1}^{k-1}(1-\beta_t),
    &G = \sum_{k=1}^\infty \pi_k\delta_{\theta_k},
\end{align}

\noindent where $\mathcal{B}(\cdot|1,\alpha)$ is a Beta distribution with parameters $1$ and $\alpha>0$. Then $G\sim \text{DP}(\alpha,G_0)$. We can understand this with the stick-breaking metaphor: we break a stick of length $1$ in two parts, $\beta_1$ and $1-\beta_1$. We define $\pi_1$ with $\beta_1$ as in (\ref{stickb}) and continue breaking $1-\beta_1$ to obtain $\beta_2, \beta_3,\ldots$ and $\pi_2, \pi_3,\ldots$. This upholds the interpretation of $G$ as an infinite mixture of point masses with normalized weights $\pi_i$.

Now assume we have data $\bd{x} = \{x_1,\ldots,x_N\}$ drawn from some unknown distribution. We conceive the unknown distribution as a mixture model so that each $x_i$ has distribution $p(\cdot | \theta_i)$, where the mixing distribution over the $\theta_i$ is $G\sim \text{DP}(\alpha, G_0)$. Formally, the resulting mixture model has the following hierarchical form:
    \begin{align*}
        &G\sim \text{DP}(\alpha, G_0)\\
        &\theta_i\sim G\\
        &x_i\sim p(\cdot|\theta_i).
    \end{align*}

\noindent Consequently, this hierarchical model yields an infinite mixture model
$$p(x)=\sum_{k=1}^\infty \pi_k p(x|\theta_k),$$ where $\pi_k$ and $\theta_k$ are the weight (mixing proportion) and the parameters of the $k$-th component of the mixture, and $G$ is the prior over the parameters. Since $G$ has the clustering property mentioned above, different $x_i$'s can pertain to the same component.

If we introduce a new latent variable $z_n$ that indicates the mixture component to which $x_n$ belongs, our model can now be described by the following generative process:
\begin{align}\label{SBConstruction}
        &\text{Draw } \beta_k\sim \mathcal{B}(\cdot|1,\alpha) \text{ for } k=1,2,\ldots\nonumber\\
        &\text{Draw } \theta_k\sim G_0 \text{ for } k=1,2,\ldots\nonumber\\
        &\text{For the }n-\text{th data point:}\\
        &\hspace{10mm}\text{Draw } z_n \sim \text{Mult}(\bd{\pi})\nonumber\\
        &\hspace{10mm}\text{Draw } x_n \sim p(\cdot| \theta_{z_n}),\nonumber
\end{align}

\noindent where $z_n$ takes value $i$ with probability $\pi_i$ and $\bd{\pi}=(\pi_i)_{i=1}^\infty$ is computed using $\bd{\beta}=(\beta_i)_{i=1}^\infty$ as in (\ref{stickb}). 
The joint probability density of the DPM model is then
\begin{equation}\label{Model}
    p(\bd{x}, \bd{\beta}, \bd{z},\bd{\theta}) = \prod_{n=1}^N p(x_n|\theta_{z_n})p(z_n|\bd{\pi})\prod_{k=1}^\infty G_0(\theta_k)\mathcal{B}(\beta_k|1,\alpha).
\end{equation}

The DP prior over mixture parameters leads to an infinite mixture model. However, since $\pi_i$ tends to decrease exponentially as $i$ increases, only a finite number of clusters, $K$, are actually involved when we deal with finite datasets. This solves the problem of determining the number of components of the mixture model, as we let the DPM infer it from the data.

\subsection{Variational inference}\label{VI}

From now on, we assume that all distributions considered are conditionally exponential, and we consider only conjugate priors.

VI has been the fundamental learning procedure for DPMs since the seminal paper \cite{blei2006variational}. Given observed data $\bd{x}=\{x_1,\ldots,x_N\}$ and a model with global variables $\bd{\eta}=\{\eta_1,\ldots,\eta_K\}$ and local variables $\bd{z}=\{z_1,\ldots,z_N\}$, VI conceives the approximation of the intractable posterior $p(\bd{\eta, z}|\bd{x})$ as a continuous optimization problem \cite{bishop2006pattern}.
More precisely, VI indirectly solves
\begin{equation}\label{min_KL}
   \underset{q\in \mathcal{Q}}{\arg \min}\text{ KL}\big[q(\bd{\eta, z})||p(\bd{\eta, z}|\bd{x})\big]
\end{equation}

by solving the equivalent 
\begin{equation}\label{max_L}
    \underset{q\in \mathcal{Q}}{\arg \max}\,\mathcal{ L}(q),
\end{equation}

where $\mathcal{L}(q)$ is called Evidence Lower BOund (ELBO) and takes the form
\begin{equation}\label{def_elbo}
\mathcal{L}(q):=\int_{\bd{\eta, z}} q(\bd{\eta, z}) \log\Big(\frac{p(\bd{x},\bd{\eta, z})}{q(\bd{\eta, z})}\Big)d\bd{\eta}d\bd{z}.
\end{equation}
In this paper, we consider \textit{mean-field} VI, where the \textit{variational distributions} $q(\bd{\eta, z})\in \mathcal{Q}$ are factorizes in terms of marginals as follows:
\begin{equation}\label{mean_field_var}
 q(\bd{\eta},\bd{z}|\bd{\bd{\phi}, \lambda}) = \prod_{n=1}^N q(z_n|\phi_n) \prod_{k=1}^{K} q(\eta_k|\lambda_k),
\end{equation}
\noindent where $\{\phi_1,\ldots,\phi_N,\lambda_1\ldots,\lambda_K\}$ are the \textit{variational parameters}. We refer to \cite{blei2017variational} for a survey of VI methods.

Substituting (\ref{mean_field_var}) in (\ref{def_elbo}), the solution to (\ref{max_L}) takes the form
\begin{equation}\label{eq2}
\begin{aligned}
    &q^*(\eta_k) \propto \exp\big(\mathbb{E}_{q_{\eta_{-k}}}[\log p(\bd{x},\bd{\eta},\bd{\phi})]\big),\\
    &q^*(z_n) \propto \exp\big(\mathbb{E}_{q_{z_{-n}}}[\log p(\bd{x},\bd{\eta},\bd{\phi})]\big)
\end{aligned}
\end{equation}
\noindent where $\mathbb{E}_{q_{\eta_{-k}}}$ stands for the expectation with respect to 
\begin{equation}
    \prod_{n=1}^N q(z_n|\phi_n) \prod_{\begin{subarray}{c} k'=1\\ k'\neq k\end{subarray}}^{K} q(\eta_k'|\lambda_k').
\end{equation}

The solutions in (\ref{eq2}) are updated iteratively using a coordinate ascent algorithm \cite{bishop2006pattern} to obtain the solution to (\ref{min_KL}). From now on, we write the ELBO as $\mathcal{L}(\bd{\lambda, \phi}|\bd{x}, \bd{\lambda}_0)$ to emphasize its dependency on the variational parameters and the data, and $\bd{\lambda}_0$ refers to the natural parameters of $G_0$.

\subsection{Collapsed variational inference for DPMs}
In the model we described in (\ref{SBConstruction}) there is a strong dependency between $\bd{z}$ and $\bd{\beta}$ via $\bd{\pi}$. However, the mean-field VI assumes that all the parameters are independent from each other (\ref{mean_field_var}). Following \cite{teh2006collapsed} and \cite{kurihara2007collapsed}, by marginalizing the mixture weights of the DPM the model presented in (\ref{Model}) becomes

\begin{equation}\label{CollapsedModel}
    p(\bd{x}, \bd{z},\bd{\theta}) = p(\bd{z})\prod_{n=1}^N p(x_n|\theta_{z_n})\prod_{k=1}^\infty p(\theta_k).
\end{equation}

This procedure of marginalizing out the mixture weights inproves the ELBO \cite{teh2006collapsed}, and hence the variational posterior in comparison to standard stick-breaking VI. 
Observe that (\ref{CollapsedModel}) has an infinite number of global parameters, which makes straightforward variational inference unfeasible. Then, following \cite{blei2006variational}, we can truncate the variational distribution to the first $K$ global parameters $\theta_1,\ldots,\theta_K$, and $p(\bd{z})$ takes the form
\begin{equation}
    p(\bd{z}) = \prod_{k<K}\frac{\Gamma(1+N_k)\Gamma(\alpha + N_{>k})}{\Gamma(1+\alpha+N_{\geq k})},
\end{equation}
where $N_k = \sum_{n=1}^N\mathds{1}(z_n=k)$, $N_{>k} = \sum_{n=1}^N\mathds{1}(z_n>k)$ and $N_{\geq k} = N_k + N_{>k}$, where $\mathds{1}$ is the indicator function.

Substituting $$q(\bd{\theta},\bd{z}|\bd{\lambda},\bd{\phi}) = \prod_{n=1}^N q(z_n|\phi_n) \prod_{k=1}^{K} q(\theta_k|\lambda_k)$$ and (\ref{CollapsedModel}) in (\ref{def_elbo}) we find the explicit ELBO for collapsed VI in DPMs:

\begin{equation}\label{ELBO_Collapsed}
\begin{aligned}
        \mathcal{L}(\bd{\lambda},\bd{\phi}|\bd{x}) =& \sum_{n=1}^N\mathbb{E}_q[\log p(x_n|\theta_{z_n})]\\
        +& \sum_{k=1}^K\Big(\mathbb{E}_q[\log p(\theta_k)] - \mathbb{E}_q[\log q(\theta_k|\lambda_k)]\Big)\\
        +& \mathbb{E}_q[\log p(\bd{z})] - \sum_{n=1}^N\mathbb{E}_q[\log q(z_n|\phi_n)].
\end{aligned}
\end{equation}

Using (\ref{eq2}), the update equations for the variational distributions become
\begin{equation}\label{update_global_dpm}
\begin{aligned}
    &q^*(\theta_k|\lambda_k) \propto p(\theta_k)\exp \Big( \sum_{n=1}^N q(z_n = k)\log p(x_n|\theta_k)\Big),
\end{aligned}
\end{equation}

\begin{equation}\label{update_local_dpm}
\begin{aligned}
    q^*(z_n|\phi_k) &\propto \exp\Big( \mathbb{E}_{q_{z_n}}[\log p(z_n| \bd{z}_{-n})]\Big)\\&\times
    \exp\Big( \mathbb{E}_{q_{\theta_{z_n}}}[\log p(x_n| \theta_{z_n})] \Big),
\end{aligned}
\end{equation}
where we write $q_{z_n}$ to denote $q(z_n)$ and so on.

Every term on these update equations can be directly computed except for $\mathbb{E}_{q_{z_n}}[\log p(z_n| \bd{z}_{-n})]$, which involves assignments of $\bd{z}$ growing exponentially with $N$. Following \cite{kurihara2007collapsed} and \cite{huynh2016streaming}, we approximate this term via its second order Taylor expansion. The first order approximation is known to be enough in many contexts \cite{sato2012rethinking, asuncion2009smoothing}, and we use the second order one for extra precision.

The model proposed in this work is an adaptation of this last model. The adaptation consists of a forgetting mechanism that is able to deal with non-stationary data streams.

\section{Streaming VI for non-stationary DPMs}\label{proposal}

A data stream can be represented as a sequence of batches of points $\bd{x}_t \in\mathbb{R}^{d \times N}$ for $t>0$, where $t$ corresponds to the time stamp of the batch, $d$ is the dimensionality of each point and $N$ is the size of every batch. We say that a concept drift occurs when the underlying distribution of data changes. 

Streaming variational Bayes (SVB) is the best known adaptation of VI to the streaming scenario \cite{broderick2013streaming}. At time $t\geq 1$ we receive the data batch $\bd{x}_t$, and we have to solve 
\begin{equation}\label{standard_SVB}
       \underset{\bd{\lambda}_t,\bd{\phi}_t}{\text{arg max }} \mathcal{L}(\bd{\lambda}_t,\bd{\phi}_t|\bd{x}_t, \bd{\lambda}_{t-1}),
\end{equation}
where $\bd{\lambda}_{t-1}$ are the variational global parameters inferred in the previous batch. Thus, the global posterior for batch $t-1$, $q(\cdot|\bd{\lambda}_{t-1})$, is used as a prior for batch $t$. This approach assumes data interchangeability and this assumption is not appropriate for non-stationary data streams. 

\subsection{SVB with power priors}

In this work, following \cite{karny2014approximate} and \cite{ozkan2013marginalized}, we propose as a prior for batch $t$ the combination of the uninformative prior $G_0(\bd{\theta}_t)$ and  $q(\bd{\theta}_{t}|\bd{\lambda}_{t-1})$ using an exponential forgetting mechanism:
\begin{equation}\label{exp_forg_prior}
    \hat{p}(\bd{\theta}_t|\bd{\lambda}_{t-1},\rho_t)\propto q(\bd{\theta}_t|\bd{\lambda}_{t-1})^{\rho_t} G_0(\bd{\theta}_t)^{1-\rho_t},
\end{equation}
where $\rho_t\in [0,1]$ is the forgetting parameter for batch $t$. Hence when $\rho_t=1$, we recover standard SVB in (\ref{standard_SVB}) and when $\rho_t=0$ we simply carry out batchwise VI. Intermediate values of $\rho_t$ emphasize either preserving previous information or reseting the prior.  

By taking $G_0(\bd{\theta}_t)$ from the same exponential family as $q(\bd{\theta}_{t}|\bd{\lambda}_{t-1})$, we have that $\hat{p}(\bd{\theta}_t|\bd{\lambda}_{t-1},\rho_t)$ remains in that family, with natural parameters
\begin{equation}\label{nat_param_exp}
    \rho_t\bd{\lambda}_{t-1}+(1-\rho_t)\bd{\lambda}_0,
\end{equation}
where $\bd{\lambda}_0$ is the natural parameter of the prior $G_0(\bd{\theta})$ \cite{kulhavy1996duality}.

This procedure is similar to the \textit{stabilized forgetting} proposed by \cite{quinn2007learning} for Dirichlet processes. However, they use MCMC methods for the inference and a linear forgetting mechanism, so the base family in each batch changes. As far as we know, our proposal is the first to introduce exponential forgetting in VI for DPMs, and allows SVB-style updates from posteriors to priors.

To wrap up, using the power priors method and choosing proper exponential family distributions, we introduce a forgetting parameter in our inference framework. We will solve the following VI problem in each batch, where only the prior differs from (\ref{standard_SVB}):
\begin{equation}\label{new_var_prob}
    \bd{\lambda}_t, \bd{\phi}_t=\arg \max \limits_{\bd{\lambda}_t,\bd{\phi}_t} \mathcal{L}(\bd{\lambda}_t,\bd{\phi}_t|\bd{x}_t,\rho_t\bd{\lambda}_{t-1}+(1-\rho_t)\bd{\lambda}_0).
\end{equation}

Algorithm \ref{PP-Alg} summarizes power prior inference for Dirichlet process mixtures. In this algorithm we have a single forgetting parameter that is manually chosen.

\begin{algorithm}[h]\label{PP-Alg}
\caption{PP-DPM}
\textbf{Input:} Batch $\bd{x}_t$, $\bd{\lambda}_{t-1}$ and forgetting parameter $\rho_t$\\
\textbf{Output:} $\bd{\lambda}_t, \bd{\phi}_t$

$\bd{\lambda}_t\leftarrow \bd{\lambda}_{t-1}$\\
Initialize $\bd{\phi}_t$\\
\textbf{repeat}\\
\hskip 0.5cm Compute $\hat{p}(\cdot|\bd{\lambda}_{t-1},\rho_t)$, (Eq. \ref{nat_param_exp}).\\
\hskip 0.5cm \textbf{for} $1\leq k\leq K$: \\
\hskip 1cm Compute $q(\theta_{t,k})$, (Eq. \ref{update_global_dpm}) with\\ \hskip 1cm $\hat{p}(\theta_{t,k}|\bd{\lambda}_{t-1},\rho_t)$ instead of $p(\theta_{t,k})$.\\
\hskip 0.5cm \textbf{for} $1\leq n\leq |\{\bd{x}_t\}|$: \\
\hskip 1cm Update $q(z_{t,n})$, (Eq. \ref{update_local_dpm}).\\
\textbf{until} convergence\\
\textbf{return }$\bd{\lambda}_t, \bd{\phi}_t$
\end{algorithm}

\subsection{SVB with hierarchical power priors}

In Algorithm \ref{PP-Alg} the forgetting parameter $\rho_t$ is selected by the user and, in practice, its optimal value can be difficult to find. Moreover, ideally, this parameter should change over the time in order to have a quick response to a concept drift.

In order to overcome these limitations, we automatically learn the value of the forgetting parameter $\rho_t$ with a technique based on \cite{masegosa2017bayesian}. We introduce prior and variational distributions for $\rho_t$ in the variational inference mechanism. This means that our approximation to the optimal $\rho_t$ will automatically change from batch to batch depending on the magnitude of the drift. Thus, this approach allows to detect the drifts by inspecting the values of $\rho_t$. 

We use as a prior a truncated exponential distribution with natural parameter $\gamma$:
\begin{equation}\label{rho_prior}
    p(\rho_t|\gamma) = \frac{\gamma\exp(-\gamma\rho_t)}{1-\exp(-\gamma)}.
\end{equation}

The variational distribution $q(\rho_t|\omega_t)$ will also be a truncated exponential with parameter $\omega_t$, where \begin{equation}\label{expect_rho}
    \mathbb{E}_q[\rho_t] = 1/(1-e^{-\omega_t}) - 1/\omega_t.
\end{equation}
The variational parameter $\omega_t$ has a natural interpretation in terms of forgetting: if $\omega_t<0$, then $\mathbb{E}_q[\rho_t]<0.5$ and the model favours $p_0(\bd{\theta}_t)$ as a better fit, hence forgetting more past data. Conversely, if $\omega_t>0$, $\mathbb{E}_q[\rho_t]>0.5$ and more emphasis is given to past data \cite{masegosa2020variational}.     

Plugging the prior over $\rho_t$ in our collapsed DPM model we obtain an ELBO, $\mathcal{L}_{\text{HPP}}$, in which we cannot work directly, because $\rho_t$ breaks the exponential conjugacy conditions for VI. 
\begin{equation}\label{ELBO_Collapsed_HPP}
\begin{aligned}
        &\mathcal{L}_{\text{HPP}}(\bd{\lambda}_t,\bd{\phi}_t,\omega_t|\bd{x}_t,\bd{\lambda}_{t-1}) = \sum_{n=1}^N\mathds{E}_q[\log p(x_{t,n}|\theta_{z_{t,n}})]\\
        &+ \sum_{k=1}^K\Big(\mathds{E}_q[\log \hat{p}(\theta_{t,k}|\lambda_{t-1,k},\rho_t)] - \mathds{E}_q[\log q(\theta_{t,k}|\lambda_{t,k})]\Big)\\
        &+ \mathds{E}_q[\log p(\bd{z})] - \sum_{n=1}^N\mathds{E}_q[\log q_n(z_{t,n}|\phi_{t,n})]\\
        &+ \mathds{E}_q[\log p(\rho_t|\gamma)] - \mathds{E}_q[\log q(\rho_t|\omega_t)].
\end{aligned}
\end{equation}

In order to solve this problem and following \cite{masegosa2017bayesian}, we propose the following surrogate ELBO:

\begin{equation}\label{surrogate_ELBO_Collapsed_HPP}
\begin{aligned}
        &\hat{\mathcal{L}}_{\text{HPP}}(\bd{\lambda}_t,\bd{\phi}_t,\omega_t|\bd{x}_t,\bd{\lambda}_{t-1})  = \sum_{n=1}^N\mathbb{E}_q[\log p(x_{t,n}|\theta_{z_{t,n}})]\\
        &+ \sum_{k=1}^K\Big(\mathbb{E}_q[\rho_t]\mathbb{E}_q[\log q(\theta_{t,k}|\lambda_{t-1,k})] + (1-\mathbb{E}_q[\rho_t])\mathbb{E}_q[\log G_0(\theta_{t,k})]\Big)\\
        &- \sum_{k=k}^K\mathbb{E}_q[\log q(\theta_{t,k}|\lambda_{t,k})]
        + \mathbb{E}_q[\log p(\bd{z})]\\
        &- \sum_{n=1}^N\mathbb{E}_q[\log q(z_{t,n}|\phi_{t,n})] + \mathbb{E}_q[\log p(\rho_t|\gamma)] - \mathbb{E}_q[\log q(\rho_t|\omega_t)].
\end{aligned}
\end{equation}

Now, we can maximize $\hat{\mathcal{L}}_{\text{HPP}}$ with respect to the variational parameters $\bd{\lambda}_t,\bd{\phi}_t, \omega_t$ using efficient updating rules. According to the following theoretical result, we have that optimization over the surrogate $\hat{\mathcal{L}}_{\text{HPP}}$ provides the same results for $\bd{\lambda}_t,\bd{\phi}_t$ as the standard VI in $\mathcal{L}_{\text{HPP}}$.

\begin{theorem}[Properties of $\hat{\mathcal{L}}_{\text{HPP}}$]\label{thm_1}{\ }
\begin{enumerate}
        \item $\hat{\mathcal{L}}_{\text{HPP}}$ is a lower bound of $\mathcal{L}_{\text{HPP}}$.
        \item  $\hat{\mathcal{L}}_{\text{HPP}}-\mathcal{L}_{\text{HPP}}$ does not depend on $\bd{\lambda}_t$ and $\bd{\phi}_t$, thus maximizing $\hat{\mathcal{L}}_{\text{HPP}}$ w.r.t. them also maximizes $\mathcal{L}_{\text{HPP}}$.
    \end{enumerate}
\end{theorem}
\begin{proof}
The proof relies on Jensen's inequality and the convexity of the log-normalizer of $\hat{p}(\bd{\theta}_t|\bd{\lambda}_{t-1},\rho_t)$. See \cite{masegosa2017bayesian}, Theorem 1.
\end{proof}

This result indicates that the update rules for $\bd{\lambda}_t,\bd{\phi}_t$ are the same as in (\ref{update_global_dpm}) and (\ref{update_local_dpm}). In order to update $\omega_t$, we use the natural gradient \cite{amari1998natural} of  $\hat{\mathcal{L}}_{\text{HPP}}$ with respect to $\omega_t$. This results in the update rule
\begin{equation}\label{optimal_omega}
    \omega_t^* = \text{KL}\big(q(\bd{\theta}_t|\bd{\lambda}_t)||G_0(\bd{\theta}_t)\big) - \text{KL}\big(q(\bd{\theta}_t|\bd{\lambda}_t)||q(\bd{\theta}_t|\bd{\lambda}_{t-1})\big) + \gamma
\end{equation}
The following pseudo-code summarizes the hierarchical power prior method for DPMs:
\begin{algorithm}[h]\label{HPP-Alg}
\caption{HPP-DPM}
\textbf{Input:} Data batch $\bd{x}_t$ and variational posterior $\bd{\lambda}_{t-1}$.\\
\textbf{Output:} $\bd{\lambda}_t, \bd{\phi}_t, \omega_t$.

$\bd{\lambda}_t\leftarrow \bd{\lambda}_{t-1}$\\
$\mathbb{E}_q[\rho_{t}] = 1/2$\\
Initialize $\bd{\phi}_t$.\\
\textbf{repeat}\\
\hskip 0.5cm Compute $\hat{p}(\cdot|\bd{\lambda}_{t-1}, \mathbb{E}_q[\rho_{t}])$ (Eq. \ref{exp_forg_prior}).\\
\hskip 0.5cm \textbf{for} $1\leq k\leq K$: \\
\hskip 1cm Compute $q(\theta_{t,k})$ (Eq. \ref{update_global_dpm}) with\\ \hskip 1cm $\hat{p}(\theta_{t,k}|\bd{\lambda}_{t-1}, \mathbb{E}_q[\rho_{t}])$ instead of $p(\theta_{t,k})$.\\
\hskip 0.5cm \textbf{for} $1\leq n\leq |\{\bd{x}_t\}|$: \\
\hskip 1cm Update $q(z_{t,n})$ (Eq. \ref{update_local_dpm}).\\
\hskip 0.5cm Compute $w_{t}$ (Eq. \ref{optimal_omega}).\\
\hskip 0.5cm Update $\mathbb{E}_q[\rho_{t}]$ (Eq. \ref{expect_rho}).\\
\textbf{until} convergence.\\
\textbf{return }$\bd{\lambda}_t, \bd{\phi}_t, \omega_t.$
\end{algorithm}

\subsection{Multiple hierarchical power priors for DPMs}
Algorithm \ref{HPP-Alg} considers a single forgetting parameter $\rho_t$. This approach can be extended by considering one independent power prior $\rho_{t,k}$ for each global parameter $\theta_{t,k}$ of the mixture model. This can be easily done by assuming that the $\rho_{t,k}$ are pairwise independent. With this assumption, we implicitly consider that the components of a non-stationary infinite mixture have different dynamics. The update rule for each $\omega_{t,k}$ associated to the parameter of the mixture $\theta_{t,k}$ is:

\begin{equation}\label{m_optimal_omega}
\begin{aligned}
    \omega_{t,k}^* =& \text{KL}\big(q(\theta_{t,k}|\lambda_{t,k})||G_0(\theta_{t,k})\big)\\
    &- \text{KL}\big(q(\theta_{t,k}|\lambda_{t,k})||q(\theta_{t,k}||\lambda_{t-1,k})\big) + \gamma
\end{aligned}
\end{equation}

This extension allows the model to have different forgetting mechanisms for each global parameter, and will be crucial for our DPM model, since the concept drift does not necessarily affect every mixture component equally. We refer to this model as \textit{multiple hierarchical power priors} (MHPP). The inference mechanism of MHPP is shown in Algorithm \ref{MHPP-Alg}.

\begin{algorithm}[h]\label{MHPP-Alg}
\caption{MHPP-DPM}
\textbf{Input:} Data batch $\bd{x}_t$ and variational posterior $\bd{\lambda}_{t-1}$\\
\textbf{Output:} $\bd{\lambda}_t, \bd{\phi}_t, \bd{\omega}_t$

$\bd{\lambda}_t\leftarrow \bd{\lambda}_{t-1}$\\
$\mathbb{E}_q[\rho_{t,k}] = 1/2$ for $1\leq k \leq K$\\
Initialize $\bd{\phi}_t$\\
\textbf{repeat}\\
\hskip 0.5cm \textbf{for} $1\leq k\leq K$: \\
\hskip 1cm Compute $\hat{p}(\cdot|\lambda_{t-1,k},\mathbb{E}_q[\rho_{t,k}])$, (Eq. \ref{exp_forg_prior}).\\
\hskip 1cm Compute $q(\theta_{t,k})$, (Eq. \ref{update_global_dpm}) with\\ \hskip 1cm $\hat{p}(\theta_{t,k}|\lambda_{t-1,k},\mathbb{E}_q[\rho_{t,k}])$ instead of $p(\theta_{t,k})$.\\
\hskip 1cm Compute $w_{t,k}$, (Eq. \ref{m_optimal_omega}).\\
\hskip 1cm Update $\mathbb{E}_q[\rho_{t,k}]$, (Eq. \ref{expect_rho}).\\
\hskip 0.5cm \textbf{for} $1\leq n\leq |\{\bd{x}_t\}|$: \\
\hskip 1cm Update $q(z_{t,n})$, (Eq. \ref{update_local_dpm}).\\
\textbf{until} convergence\\
\textbf{return }$\bd{\lambda}_t, \bd{\phi}_t, \bd{\omega}_t$
\end{algorithm}

\section{Experiments}\label{experim}

In this section we empirically evaluate the three proposed models: PP (Algorithm \ref{PP-Alg}), where the single forgetting parameter has to been hand-tuned beforehand; HPP (Algorithm \ref{HPP-Alg}), which automatically learns the forgetting parameter; and MHPP-DPM (Algorithm \ref{MHPP-Alg}), which learns a forgetting parameter for each global parameter of the mixture model. For the PP methods, in each experiment we choose the best forgetting parameter in $\rho\in \{0.9, 0.99\}$.

We compare their performance against the following baselines:
\begin{itemize}
    \item Streaming variational bayes DPM (SVB) of \cite{broderick2013streaming}.
    \item Stochastic variational inference (SVI) of \cite{hoffman2013stochastic} as implemented by \cite{hughes2014bnpy}.
    \item Privileged-DPM (Privileged), a version of SVB-DPM which discards previous information when a drift happens.
\end{itemize}
SVB and SVI are the state-of-the-art procedures for DPM over data streams. Privileged represents the gold standard, however, it can not be used in practice because it requires to  know when each concept drift occurs. A Python implementation of SVB, SVI, Privleged, PP, HPP and MHPP will be available online in an open repository after publication. For now, it can be found in the anonymous repository \url{https://anonymous.4open.science/r/DPM-for-non-stationary-streams-0E44/README.md}

In HPP and MHPP, the inference step for $\rho_t$ is simultaneous to that of $\bd{\theta}$ and $\bd{z}$, hence the complexity will be the same for every model based on VI.

We evaluate the ability to adapt to these drifts in the following tasks: 
\begin{itemize}
    \item Density estimation: we measure the quality of the learned DPMs using the log-likelihood in test data.
    \item Clustering: we evaluate the obtained clustering using four popular clustering metrics: Silhouette score, \cite{ROUSSEEUW198753}, Normalized Mutual Information (NMI), Adjusted Rand Index (ARI) and Purity. We have reported the mean values across all the batches of the data stream. All the measures have been computed using the implementation of the library \cite{scikit-learn}.
    \item Model parameter tracking: using the synthetic data, we compare the parameters of the mixture model obtained by the different algorithms with respect to the parameters of the true model.
\end{itemize}

\begin{figure*}[t]
\subfloat[Drift every 4 batches]{
\includegraphics[width=\columnwidth]{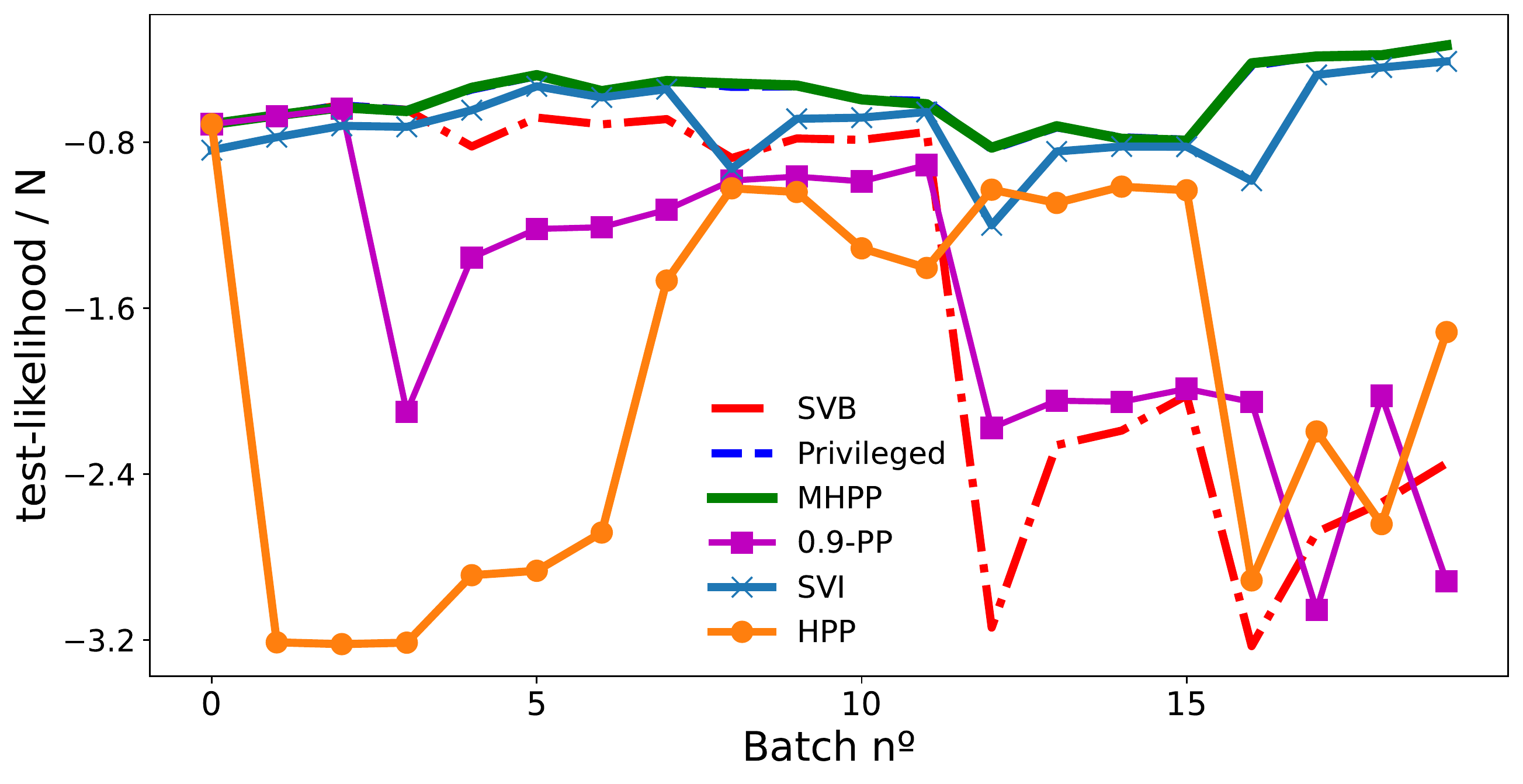}\label{fig:Synth_lik_2}}
\subfloat[Drift every batch]{
\includegraphics[width=\columnwidth]{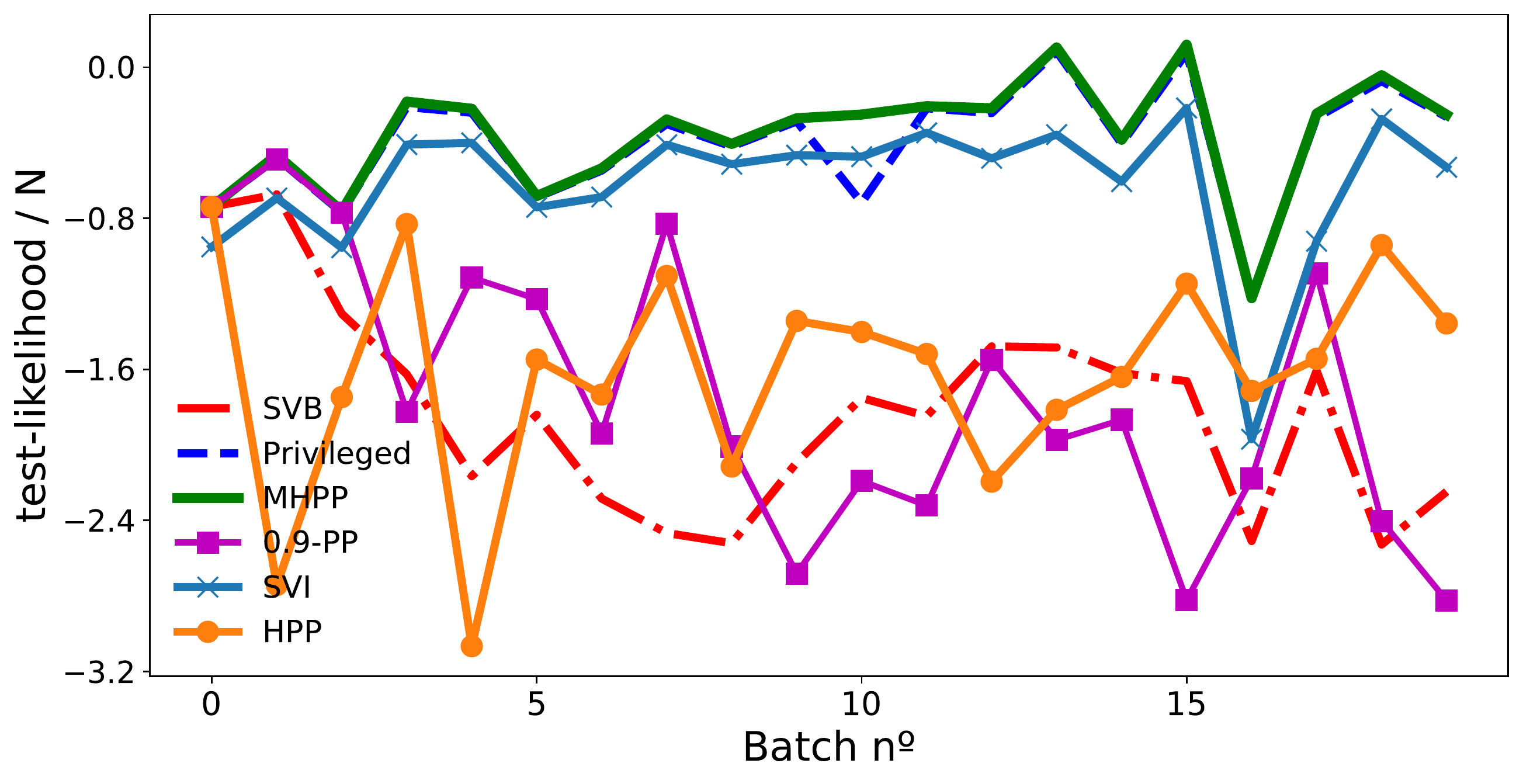}\label{fig:Synth_lik}}
\caption{log-likelihood per test data point from 2D Gaussians. We consider different drift frequencies.}
\end{figure*}

\begin{table*}[h]
\centering
\begin{tabular}{*{7}{||c| c| c c c c ||c}}
 \hline
  & \textit{Privileged} & SVB & SVI & $0.9$-PP & HPP & \textbf{MHPP} \\
 \hline\hline
  \multicolumn{7}{|c|}{Drift every 4 batches}\\ \hline
 Silhouette score & $0.83\pm 0.05$ & $\bd{0.81}\pm 0.07$ & $0.68\pm 0.12$ & $0.55\pm 0.29$ & $0.71\pm 0.15$ &  $0.80\pm 0.10$ \\
 \hline
 NMI score & $1$ & $\bd{0.99}\pm 0.01$ & $0.83\pm 0.19$ & $0.87\pm 0.16$ &$0.96\pm 0.05$ & $\bd{0.99}\pm 0.01$ \\ 
 \hline
 ARI score & $1$ & $\bd{0.99}\pm 0.01$ & $0.95\pm 0.07$ & $0.75\pm 0.27$ & $0.81\pm 0.24$  & $\bd{0.99}\pm 0.02$ \\
 \hline
 Purity score & $1$ & $\bd{1}$ & $0.98\pm 0.04$ & $0.84\pm 0.19$ & $0.83\pm 0.21$  & $\bd{1}$ \\
 \hline\hline
  \multicolumn{7}{|c|}{Drift every batch}\\ \hline
 Silhouette score & $0.84\pm 0.03$ & $0.74\pm 0.10$ & $0.57\pm 0.23$ & $0.46\pm 0.28$ & $0.78\pm 0.08$ & $\bd{0.82}\pm 0.06$ \\
 \hline
  NMI score & $0.99\pm 0.04$ & $0.94\pm 0.07$ & $0.95\pm 0.06$ & $0.79\pm 0.03$  & $0.91\pm 0.09 $ & $\bd{0.99}\pm 0.03$ \\ 
 \hline
ARI score & $0.97\pm 0.09$ & $0.92\pm 0.11$ & $0.94\pm 0.09 $ & $0.67\pm 0.27$ & $0.84\pm 0.15$  & $\bd{0.98}\pm 0.06$ \\
 \hline
 Purity score & $0.97\pm 0.07$ & $0.96\pm 0.09$ &  $0.86\pm 0.6$  & $0.77\pm 0.19$ & $0.86\pm 0.15$ & $\bd{0.99}\pm 0.05$ \\
 \hline
\end{tabular}
\caption{Results for different cluster metrics in 2D Gaussians.\\ We do not consider \textit{Privileged} when highlighting the best algorithm}
\label{table:table_syn}
\end{table*}

In these experiments every algorithm implements collapsed Dirichlet process (isotropic) Gaussian mixtures with truncation $T$, hence in our case the global parameters are the means and covariances of each component: 
\begin{equation}\label{dpgm_comp}
    \bd{\theta} = \{\bd{\mu}_1,\tau_1,\ldots \bd{\mu}_T, \tau_T\}.
\end{equation} We use uninformative conjugate priors $\mathcal{N}(\bd{\mu}; \bd{0}, I)$, $\text{Gamma}(\tau; 1,1)$. 
\subsection{Synthetic data}

In this experiment we generate four 2D isotropic Gaussians, and we randomly vary their mean and covariance for 20 batches in order to simulate drift.
For this experiment we set $\alpha = 2$, truncation parameter $T = 10$ and run all the algorithms for $100$ iterations. We use 1000 training points and 500 test points per batch. In the case of SVI, we work as if each batch was the full dataset and warm-start the model for the next batch. This can bias the results towards SVI. We fix its learning parameters $rhoexp = 0.55$ and $rhodelay = 1$. The experiments have been repeated 10 times.

\begin{figure*}[t]
\subfloat{\includegraphics[width=\columnwidth]{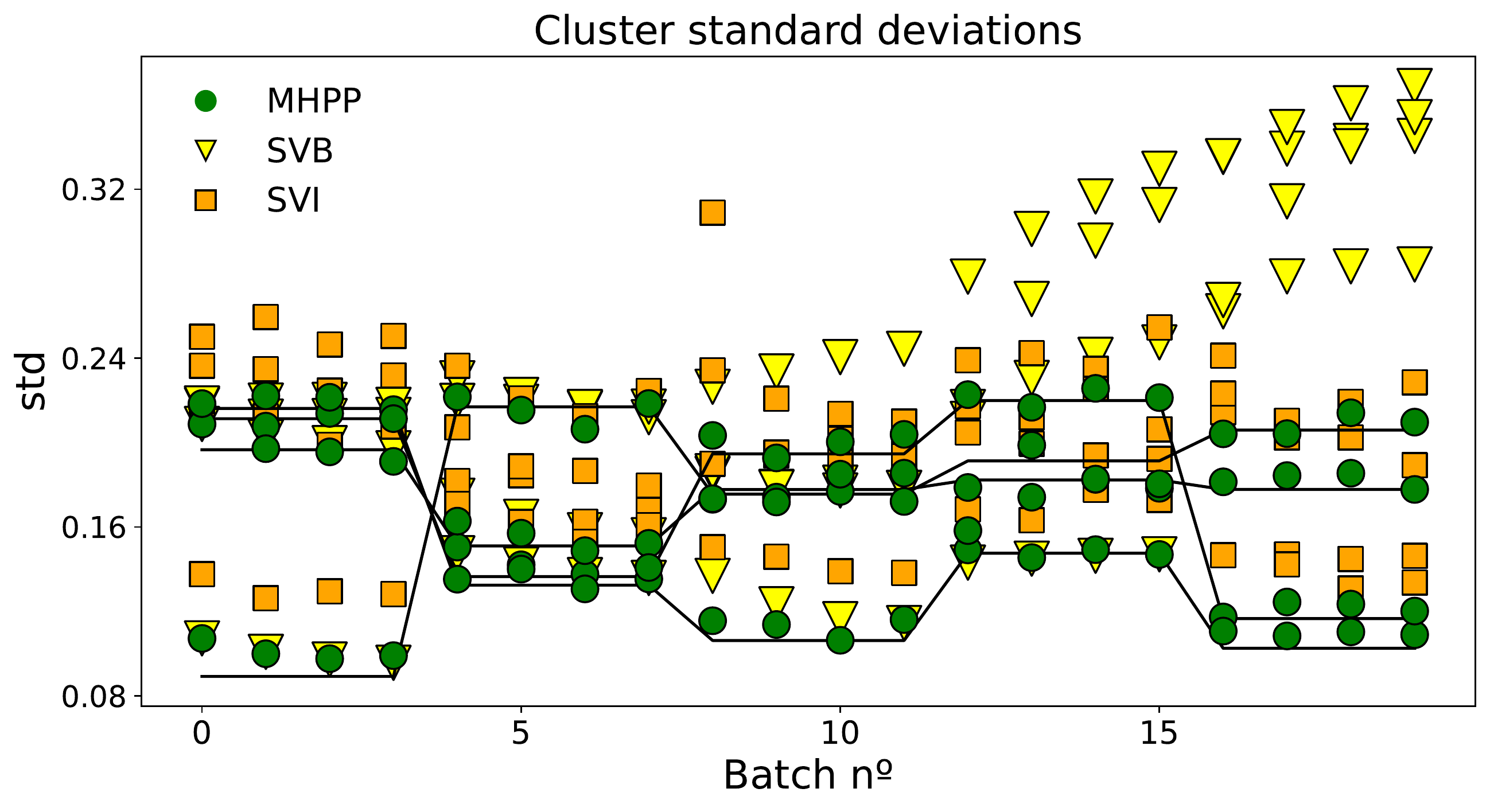}\label{fig:param_track_1}}
\subfloat{\includegraphics[width=\columnwidth]{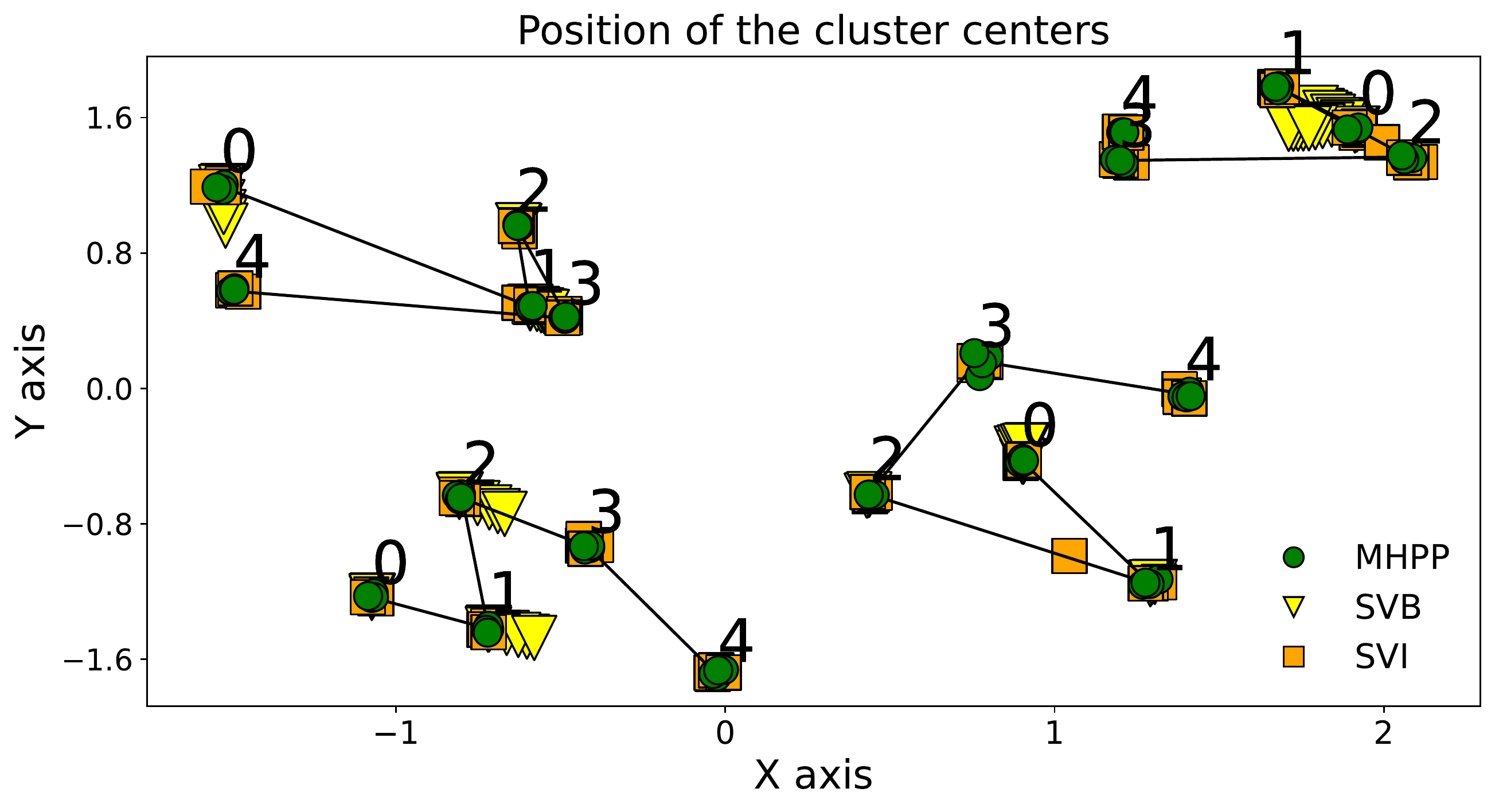}\label{fig:param_track_2}}
\caption{Global parameter tracking for different algorithms. We represent data of the 4 most populated components. Ground truth is indicated by black lines.}
\end{figure*}
\subsubsection{Density estimation}

Figures \ref{fig:Synth_lik_2} and \ref{fig:Synth_lik} compare the algorithms according to the log-likelihood (higher is better) of test data using the obtained Gaussian mixture. In \ref{fig:Synth_lik_2}, where concept drift occurs every $4$ batches, the performance of MHPP is remarkably similar to that of Privileged, both in response to drifts and in the stationary phase. MHPP outperforms the state-of-the-art procedures. HPP and 0.9-PP obtain worse results, which shows that the use of a single parameter is insufficient to deal with the synthetic data.

Figure \ref{fig:Synth_lik} shows the same experimental framework with drift in every batch. Again, MHPP obtains results remarkably similar to Privileged. This shows that MHPP is able to address density learning in scenarios with very frequent drifts. Conversely, PP, SVB and SVI methods perform worse with frequent drifts. Again, HPP and 0.9-PP show high numerical instability, and its performance justifies the need for multiple forgetting parameters as in MHPP.

\subsubsection{Clustering}

Table \ref{table:table_syn} summarizes the results of different cluster metrics under the two drift frequencies studied. When the drift occurs every 4 batches MHPP and SVB obtain the best results. In the scenario where the drift occurs at every batch, MHPP obtains the best results, and they are remarkably better than the state-of-the-art. This suggests that MHPP is the best algorithm addressing the concept drift in clustering problems.

\subsubsection{Parameter tracking}

To analyze the ability of the algorithms to learn the parameters of the true Gaussian mixture model, we show the evolution of the estimated means and standard deviations of the four most populated clusters in each batch. In order simplify the visualization of the results, we have selected MHPP, SVB and SVI, and we have considered the scenario with drifts every 4 batches. Figures \ref{fig:param_track_1} and \ref{fig:param_track_2} show the evolution of the standard deviations and means respectively. In Figure \ref{fig:param_track_2} the numbers represent the order of drifts.

The results clearly show that MHPP provides the best estimation of the parameters. MHPP is able to recover the parameters of the underlying Gaussian mixture models. The proposed procedure adapts to concept drift immediately, while the time of response of SVI and SVB is higher and less accurate.

The experiments also show that the forgetting parameters of MHPP tend to $0.5$ when their component is not \textit{active} in the DPM, while capturing different dynamics for means and covariances. 

Overall, the experimentation in synthetic data upholds MHPP as the most competitive method for DPM density estimation and clustering in non-stationary streaming scenarios.

\subsection{Real data}
\begin{table*}[h]
\centering
\begin{tabular}{*{15}{||c| c| c c c c}}
 \hline
   & \textit{Privileged} & SVB & SVI & $0.99$-PP & HPP & \textbf{MHPP} \\
 \hline\hline
 Silhouette score & $0.06\pm 0.03$  & $0.02\pm 0.03$ & $\bd{0.17}\pm 0.04$  & $0.01\pm 0.04$  & $0.03\pm 0.03$ & $0.05\pm 0.02$ \\
 \hline
 NMI score & $0.67\pm 0.03$ & $0.64\pm 0.04$ & $0.24\pm 0.08$ & $0.67\pm 0.02$   &$0.67\pm 0.04$ & $\bd{0.69}\pm 0.04$ \\ 
 \hline
 ARI score & $0.45\pm 0.05$ & $0.43\pm 0.05$ & $0.08\pm 0.04$ & $0.45\pm 0.04$ & $0.48\pm 0.10$ & $\bd{0.50}\pm 0.07$ \\
 \hline
 Purity score & $0.78\pm 0.04$ & $0.75\pm 0.06$  & $0.25\pm 0.03$ & $0.76 \pm 0.05$  &$0.70\pm 0.06$ & $\bd{0.79}\pm 0.06$ \\
 \hline
\end{tabular}
\caption{Results for different cluster metrics in n-MNIST data set.\\ We do not consider \textit{Privileged} when highlighting the best algorithm}
\label{table:table_real}
\end{table*}

\begin{figure}[h]
    \centering
    \includegraphics[width=\columnwidth]{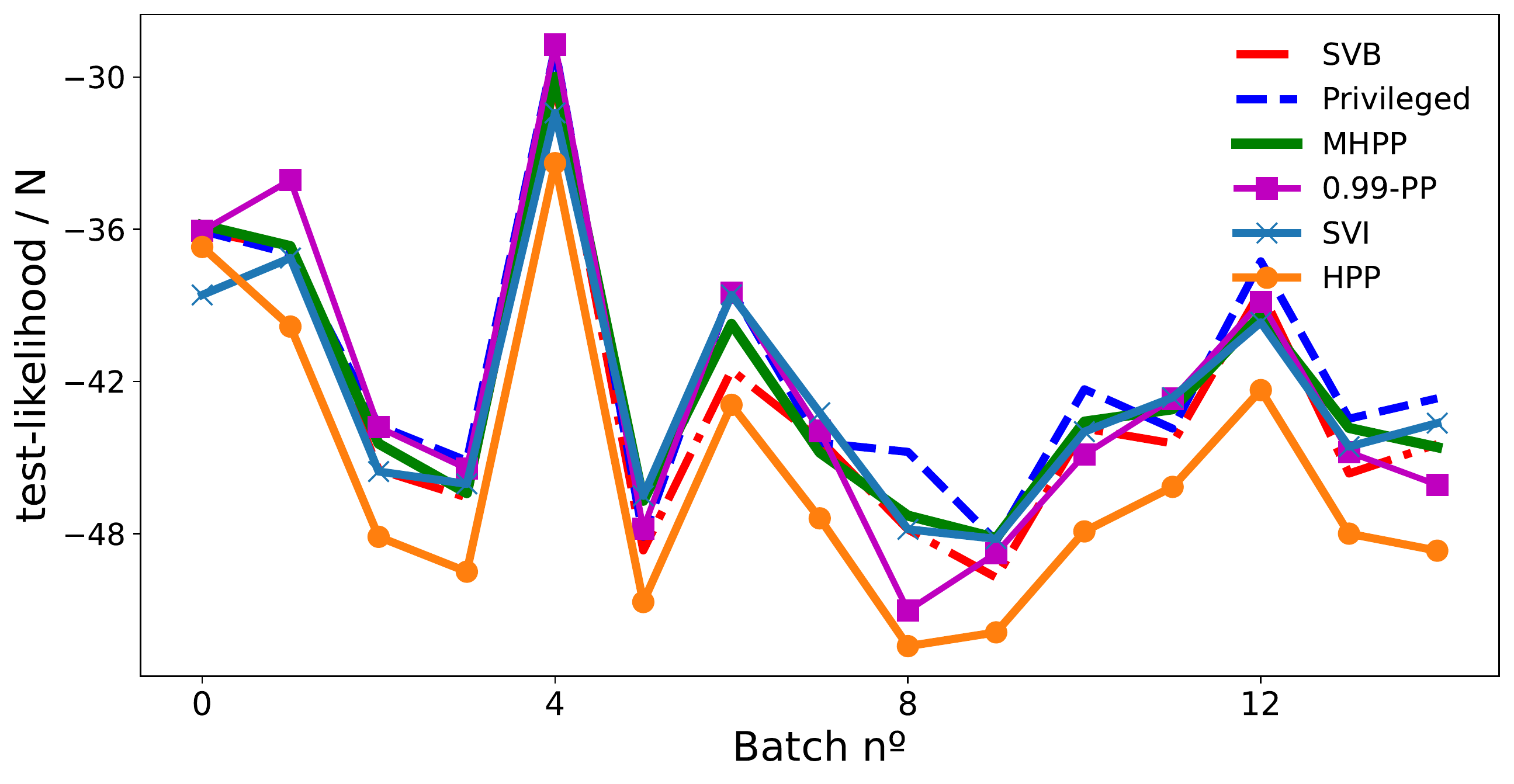}    \caption{log-likelihood per data point of different algorithms for held-out data from n-MNIST.}
    \label{fig:MNIST}
\end{figure}

In order to test our algorithm with real data, we use the MNIST \cite{lecun1998gradient} and the noisy MNIST (n-MNIST) \cite{basu2017learning} datasets. The first is the standard digit recognition dataset, while the second includes three datasets, each of them created by adding a different kind of noise to the digits: additive white Gaussian noise, motion blur, and a combination of additive white Gaussian noise and reduced contrast. The transition from MNIST to n-MNIST data and the addition or removal of type of digits will simulate the concept drifts. The experimental framework is as follows: we consider all four data sources and for each batch we first randomly select the number of digits we consider in a range from $6$ to all $9$. Then we sample those from one of the data sources randomly. This will create a data stream where the number of cluster varies and the source of those clusters can change from batch to batch.

\begin{figure}[h]
    \centering
    \includegraphics[width=8cm]{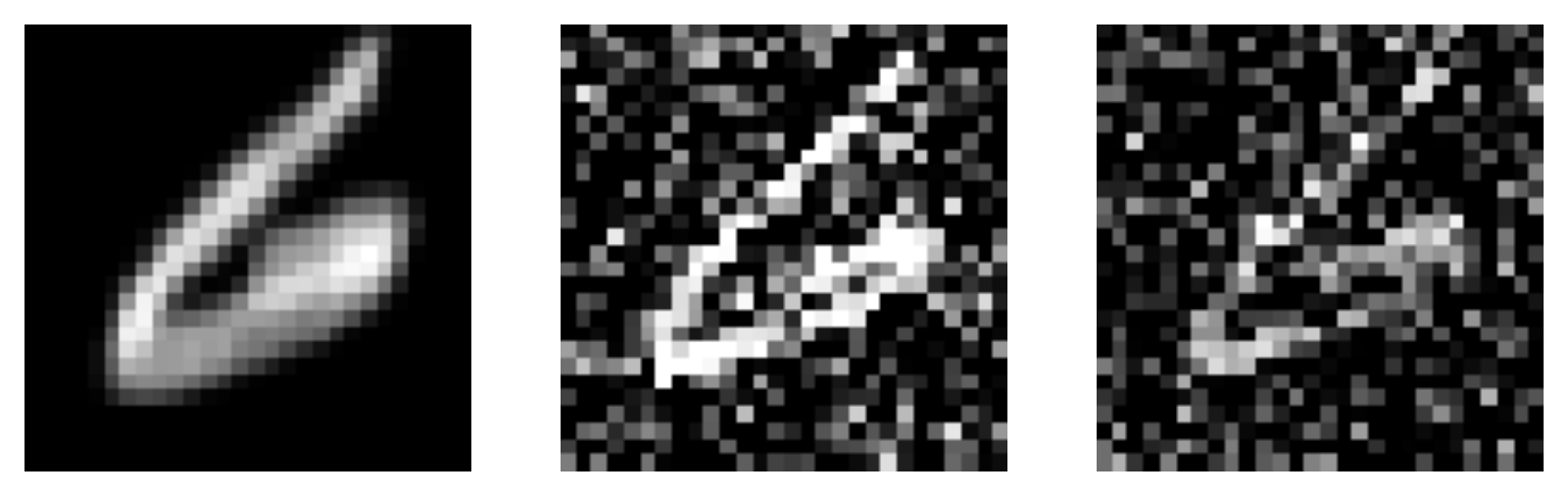}    \caption{Three n-MNIST versions of the same MNIST digit. Motion blur (left), additive gaussian noise (center) and additive gaussian noise + reduced contrast (right).}
    \label{fig:MNIST}
\end{figure}

Every dataset is preprocessed with minmax scaling and the 764 ($28\times 28$) dimensions are reduced to $50$ using PCA.
For this experiment we set truncation parameter $T=30$, $\alpha=3$, $1000$ training points and $500$ test points. 

\subsubsection{Density estimation}
Figure \ref{fig:MNIST} shows test log-likelihood for different algorithms in the n-MNIST experiment. The performances of SVI, SVB, PP and MHPP are very similar. HPP provides the worst results. Overall, all the likelihoods remain comparable. However, we have observed that SVI and SVB requires more components in the mixture model to reach the results of MHPP.

\subsubsection{Clustering}

Table \ref{table:table_real} shows the results of each model in the four clustering metrics considered. MHPP is superior in 3 out of 4 metrics, closely followed by $0.99$-PP. The different performance with respect to density estimation can be explained by the fact that the log-likelihood does not penalize the use of too many mixture components.

\section{Conclusions and future work}

In this work we propose an adaptation of Dirichlet process mixtures to streaming data with concept drifts. The proposed adaptation includes a forgetting mechanism that allows to accommodate the obtained mixture model to concept drift. We introduce three streaming variational inference methods for Dirichlet process mixtures, each of them with different forgetting mechanisms: PP uses a single forgetting parameter that needs to be hand-tuned, while HPP can automatically learn it. Finally, MHPP includes a forgetting parameter for each global parameter. MHPP is the most flexible approach because every component of the mixture model can be adapted independently to the particularities of the concept drift.

We have performed a set of experiments using synthetic and real data. The experimental results show that the forgetting parameters can be learned from data automatically. Besides, the experiments show that MHPP consistently outperforms HPP and PP, which indicates that it is advisable to have a forgetting parameter for each of the parameters of the learned model. MHPP obtains competitive results with respect to the state-of-the-art in terms of density estimation and outperforms them in the clustering problem, especially in scenarios for which the frequency of the drift is high. MHPP reacts faster to the drifts and obtains a better estimate of the parameters than the state-of-the-art. 

We plan to study the extension of MHPP to the truncation-free Dirichlet process mixtures, since proposals such as \cite{huynh2017streaming}, \cite{huynh2016streaming} could be effectively combined with our method. In addition, we will explore in more detail the interpretation of the evolution of the forgetting parameters. They will provide very useful information about the type of drifts, and it will have applications to the novelty detection problem.

\section*{Acknowledgment}


\vspace{12pt}
\bibliographystyle{IEEEtran}
\bibliography{references}
\end{document}